\documentclass[pdflatex,sn-mathphys-num]{sn-jnl}

\usepackage{hyperref}
\usepackage{graphicx}%
\usepackage{multirow}%
\usepackage{amsthm}%
\usepackage{amsmath,amssymb,amsfonts}%
\usepackage[title]{appendix}%
\usepackage{xcolor}%
\usepackage{textcomp}%
\usepackage{manyfoot}%
\usepackage{booktabs}%
\usepackage{algorithm}%
\usepackage{algorithmicx}%
\usepackage{algpseudocode}%
\usepackage{listings}%
\graphicspath{{./}}

\theoremstyle{thmstyleone}%
\newtheorem{theorem}{Theorem}
\newtheorem{lemma}{Lemma}

\theoremstyle{thmstyletwo}%

\theoremstyle{thmstylethree}%
\newtheorem{definition}{Definition}

\begin{document}

\title[Article Title]{ExO-PPO: an Extended Off-policy Proximal Policy Optimization Algorithm}


\author[1]{\fnm{Hanyong} \sur{Wang}}\email{harryw@stu.scu.edu.cn}

\author*[1]{\fnm{Menglong} \sur{Yang}}\email{mlyang@scu.edu.cn}

\affil*[1]{\orgdiv{School of Aeronautics and Astronautics}, \orgname{Organization}, \orgaddress{\street{No.24 South Section 1, Yihuan Road}, \city{Chengdu}, \postcode{610207}, \state{Sichuan}, \country{China}}}

\abstract{Deep reinforcement learning has been able to solve various tasks successfully, however, due to the construction of policy gradient and training dynamics, tuning deep reinforcement learning models remains challenging. As one of the most successful deep reinforcement-learning algorithm, the Proximal Policy Optimization algorithm (PPO) clips the policy gradient within a conservative on-policy updates, which ensures reliable and stable policy improvement. However, this training pattern may sacrifice sample efficiency. On the other hand, off-policy methods make more adequate use of data through sample reuse, though at the cost of increased the estimation variance and bias. To leverage the advantages of both, in this paper, we propose a new PPO variant based on the stability guarantee from conservative on-policy iteration with a more efficient off-policy data utilization. Specifically, we first derive an extended off-policy improvement from an expectation form of generalized policy improvement lower bound. Then, we extend the clipping mechanism with segmented exponential functions for a suitable surrogate objective function. Third, the trajectories generated by the past $M$ policies are organized in the replay buffer for off-policy training. We refer to this method as Extended Off-policy Proximal Policy Optimization (ExO-PPO). Compared with PPO and some other state-of-the-art variants, we demonstrate an improved performance of ExO-PPO with balanced sample efficiency and stability on varied tasks in the empirical experiments.}




\maketitle

\section{Introduction}\label{sec:introduction}
Deep reinforcement learning (DRL) approaches have achieved various planning and scheduling tasks, especially in those environments with complex and unstructured high dimensional stochastic visual observations. However, as the deep perception models become larger and larger, vision-language model based DRL approaches are further computationally expensive and require large amounts of training data and special tuning \cite{VRL}. As one of the most commonly used RL algorithms, Proximal Policy Optimization (PPO) \cite{PPO} usually needs millions of samples to fit one schedule. There are trade-offs between stability and sample efficiency in reinforcement learning, while on-policy algorithms prioritize stability by restricting policy consistency within updates. The PPO algorithm ensures reliable stability by clipping the policy gradient, which is derived from the Trust Region Policy Optimization \cite{TRPO} and Constrained Policy Optimization \cite{CPO}. But it limits the ability to utilize all available experiences efficiently.

Distinct from on-policy algorithms, off-policy algorithms can reduce the requirement for numerous samples by reusing data stored from previous policies in the replay buffer to calculate multiple policy updates. But the shift in distributions may cause an overestimation bias, which could easily lead to divergence from the current policy. Popular off-policy algorithms often adopt specific implementation tricks, such as target policy smoothing in Dueling DQN \cite{DDQN}, TD3 \cite{TD3}, etc., or extensive hyperparameter tuning, like SAC \cite{SAC}, to limit the bias and variance from off-policy data. In recent years, there have been several researches focused on mixing off-policy methods with the PPO algorithm. GePPO \cite{GePPO} generalized the on-policy policy improvement lower bound to depend on expectations regarding any reference policy, which enables the PPO algorithm to reuse samples from previous policies. Also, off-policy PPO \cite{Off-PPO} clarified that the use of off-policy data does not harm the stability of the clipping surrogate objective in PPO.

On the other hand, the above typical deep reinforcement learning algorithms are considered as a paradigm for online learning, where the interaction between the training agent and its environment is fundamental to evaluate how the agent learns. However, persistent interactions could also be expensive, time-consuming, or otherwise troublesome, particularly in real-world applications. This has motivated growing research efforts in offline DRL, where an agent learns from a pre-collected dataset \cite{Offrl}. Offline DRL enables to utilize previously logged data or leverage expert demonstrations, such as a human operator. However, this also makes offline RL a challenging task, as the original action distributions is not presented in the dataset. Agents tend to inaccurately estimate the value of actions not contained in the dataset \cite{MinOff,BPPO}. Regularizing the training policy with the behavior policy is one of the common approaches to addressing this problem \cite{MinOff,BPPO,staoffQ,ConOffQ,conserQ}. Another emerging research trend involves utilizing diffusion models to solve offline learning tasks. Diffuser \cite{diffuser} trains a diffusion probabilistic model that plans by iteratively denoising trajectories, Diffusion policies \cite{diffusionPolicy,EffdiffPolicy} act as the policy model, replacing conventional Gaussian policies and so on \cite{diffResurvey}. Enhancing training efficiency while maintaining policies consistency is still meaningful in this case.  

In this study, we introduce \textit{Extended Off-policy Proximal Policy Optimization (ExO-PPO)} algorithm, a novel variant of PPO that incorporates an extended off-policy sampling pattern and a specialized surrogate function. By merging the strengths of both on-policy and off-policy methods, we seek to achieve a balance between stability and sample efficiency. Additionally, a smoother surrogate objective function enables more effective gradient-based training. This surrogate function also effectively constrains the policy under offline setting. ExO-PPO maintains acceptable compatibility with the standard PPO algorithm, allowing few modifications for integration into existing workflows. Our empirical results demonstrate that ExO-PPO outperforms the PPO and some other state-of-the-art variants in various Atari games and MuJoCo tasks. Explicitly, our main contributions are as follows.

\begin{itemize}
  \item [1)] 
  We propose an extended off-policy sampling pattern by extending the on-policy pattern to prior $M$ policies from the replay buffer. Then, we derive the generalized policy improvement lower bound into an expectation form, the Extended Off-Policy Improvement Lower Bound, as the theoretical support of the extended off-policy training pattern. 

  \item [2)] 
  Due to the off-policy sampling, whose probability ratio deviates from the initial point of optimization, we re-expanding the clipped ratio with a parameterized exponential function to cover a wider gradient space. This Extended Ratio Objective can also be suitable for offline training.  

  \item [3)] 
  Empirical results demonstrate that our proposed training pattern with the specifically optimized surrogate function outperforms on various tasks, including online and offline tasks, balancing the goals of stability and sample efficiency.
\end{itemize}

\subsection{Related works}

\subsubsection{On-policy gradient methods} Policy gradient methods optimize a differentiable function approximation with a long-term return objective \cite{PolicyG}. And the actor-critic algorithm introduced a training paradigm based on policy gradient \cite{onAC}. Then, the Conservative Policy Iteration was proposed in \cite{Approx} for an approximate greedy policy improvement guarantee. Base from that guarantee, the Trust Region Policy Optimization algorithm \cite{TRPO} specialized the approximate policy surrogate objective with a Kullback-Leibler (KL) divergence trust region. Constrained Policy Optimization \cite{CPO} derived a policy optimization approach based on the trust region method, which guarantees monotonic increase in reward and constraints on other costs. Proximal Policy Optimization \cite{PPO} restricts the policy update by clipping the probability ratio between current and trained policies. 

Although the ratio clipping mechanism in PPO is a simple first-order optimization technique, it might yield a pessimistic estimate of the objective function, which could cause gradient loss at the end of each training session. Thus, there are varied researches that focus on the surrogate objective. Truly PPO and Trust Region-Guided PPO \cite{TrulyPPO,TrustRegionG} analyzed the scope of the clipping mechanism in PPO and applied a trust region recall mechanism to prevent the policy from being pushed away during training. Early Stopping Policy Optimization (ESPPO) \cite{ESPPO} showed that clipping could fail as the number of optimization steps increases in an epoch, which would stop the optimization according to the ratio deviation. Chen et al. \cite{TheSuffi} applied a sigmoid function to the ratio in the surrogate objective to accelerate the learning process. Besides, researches such as \cite{Implementationmatter} and \cite{whatmatterson} also investigated various implementation details and code-level optimizations through extensive empirical analysis of TRPO and PPO algorithms. 

\subsubsection{Off-policy gradient methods} Off-Policy Actor-Critic \cite{offPAC} presented the off-policy actor-critic algorithm, enabling temporal difference learning from behavioral policies with eligibility traces. And the Deep Deterministic Policy Gradient (DDPG) algorithm \cite{DPG,DDPG} limits the stochastic policy to be deterministic, and employs off-policy gradient to address challenges in continuous action spaces. However, such learning patterns may introduce high bias due to the distribution shift caused by off-policy data. Thus, Twin Delayed DDPG (TD3) \cite{TD3} and Soft Actor-Critic (SAC) \cite{SAC} adopt numerical tricks such as target policy smoothing or maximum entropy to mitigate the bias. 

\subsubsection{Mixed on-policy and off-policy gradient methods} With the goal of balancing the stability of on-policy methods and the sample efficiency of off-policy methods, there are methods that combine on-policy and off-policy characteristics. Interpolated Policy Gradient method \cite{Qprop, InterpolatedPG} associated the likelihood ratio gradient from an on-policy actor with the deterministic gradient through an off-policy fitted critic. And Policy-on Policy-off Policy Optimization \cite{P3O} developed an algorithm that interleaves off-policy actor gradient updates with on-policy updates. Then, Generalized PPO \cite{GePPO} derived a generalized policy improvement lower bound with clipping mechanism from on-policy improvement lower bound, supporting off-policy sampling. On the other hand, Off-Policy PPO \cite{Off-PPO} analyzed the degree of off-policy update distance, and then clarified the use of off-policy data does not harm the stability of PPO algorithm. 

In addition, as replay buffer and offline training iteration that make deep neural networks suitable for reinforcement learning \cite{DQN}, a recent study \cite{TheSuffi} argued that \textit{the nature of reinforcement learning is truly off-policy}. However, this perspective could be considered somewhat one-sided. Regarding the reference policy at training steps, the fundamental distinction between the so-called \textit{\textbf{on-policy}} and \textit{\textbf{off-policy}} methods lies in whether it is only the current one or also incorporates prior policies. 

\section{Preliminaries}\label{sec:pre}

Consider a finite discounted Markov decision process (MDP) defined by the tuple $(S,A,P,r,\gamma)$, with a state space $S$, an action space $A$, a transition probability distribution $P(s'|s,a):S \times A \times S \rightarrow [0,1]$ and a reward function $r:S \times A \times S \rightarrow \mathbb{R}$ that provides real number feedback of the transition with a discount factor $\gamma \in (0,1)$. Also, $\pi :S \times A\rightarrow[0,1]$ denotes a stochastic policy, $d_\pi(s_t)$ and $d_\pi(s_t,a_t)$ denote the state and state-action marginal distribution induced by policy $\pi$. Starting from a starting state distribution $\mu$, the discount factor $\gamma$ can be taken into the future state distribution $d_{\pi,\mu}(s)=(1-\gamma)\sum_{t=0}^{\infty}\gamma^t P(s_t=s|\mu,\pi)$. Our goal is to optimize a policy that maximize the expected total discounted reward $J(\pi)=\mathbb{E}_{(s_t,a_t)\sim d_\pi}\left[ \sum_{t=0}^\infty \gamma^t r(s_t) \right]$.

During the optimization, we also use the following standard definitions of the value function $V_\pi(s_t) = \mathbb{E}_{s_{t+1:\infty},a_{t:\infty}}\left[ \sum_{l=0}^\infty \gamma^l r(s_{t+l}) \right]$, the state-action value function $Q_\pi(s_t,a_t) = \mathbb{E}_{s_{t+1:\infty},a_{t+1:\infty}}\left[ \sum_{l=0}^\infty \gamma^l r(s_{t+l}) \right]$, and the advantage function $A_\pi(s,a) = Q_\pi(s,a)-V_\pi(s)$.

\subsection{Policy Gradient}

Policy gradient maximize the expected total reward by its first-order optimization \cite{PolicyG}:
\begin{equation}
    \frac{\partial J(\pi)}{\partial\theta}=\sum_{s}d_{\pi}(s)\sum_{a}\frac{\partial\pi(s,a)}{\partial\theta}Q_{\pi}(s,a).
\end{equation}

However, optimizing the policy gradient in terms of the state-action value $Q$ would further amplify the gradient value, potentially resulting in increased variance and even divergence. To address this issue, the advantage function, which is subtracted a baseline from $Q$, is widely adopted to decrease the variance of the gradient estimation without changing the equality. The Generalized Advantage Estimation (GAE) \cite{GAE} is an effective variance reduction technique for policy gradient optimization:
\begin{equation}
    \hat A_{t}^{{\rm GAE}(\gamma,\lambda)}=\sum_{l=0}^\infty (\gamma\lambda)^l\delta_{t+l}^V
\end{equation}
which represents an exponentially weighted average of $L$ $k$-step estimators. Then, the following on-policy gradient methods incorporate $A$ into their optimization objective functions and utilize GAE for computation during algorithmic implementation. 

\subsection{Policy Improvement Lower Bound}\label{sec:pre:PILB}
The Conservative Policy Iteration was proposed in \cite{Approx} for an approximate greedy policy improvement guarantee. Base from that guarantee, the Trust Region Policy Optimization algorithm \cite{TRPO} specialized the approximate policy surrogate objective with a Kullback-Leibler (KL) divergence trust region. Constrained Policy Optimization \cite{CPO} derived a constrained policy optimization from the trust region method that guarantees monotonic increase in reward and satisfaction of constraints on other costs. Then, the derived policy improvement lower bound can be written as \cite{CPO}:

\begin{lemma}[Policy Improvement Lower Bound]\label{lemmaLB}
    From the current policy $\pi_t$, each training policy $\pi$ has:
    \begin{equation}
        \begin{aligned}
            J(\pi)-J(\pi_t) \ge & \frac{1}{1-\gamma} \mathop{\mathbb{E}}\limits_{(s,a)\sim \pi_t} \left[ \frac{\pi(a|s)}{\pi_t(a|s)} A^{\pi_t}(s,a) \right]\\
            &- \frac{2\gamma C^{\pi,\pi_t}}{(1-\gamma)^2}\mathop{\mathbb{E}}\limits_{(s,a)\sim \pi_t}\Big[\mathrm{TV}(\pi,\pi_t)(s)\Big]
        \end{aligned}
    \end{equation}
        where $C^{\pi,\pi_t}=\max_{s\in S} \bigl| \mathbb{E}_{a\sim \pi(\cdot|s)}[A^{\pi_t}(s,a)] \bigr|$ and $\mathrm{TV}(\pi,\pi_t)(s)$ represents the total variation distance between the distributions $\pi(\cdot|s)$ and $\pi_t(\cdot|s)$.  
\end{lemma}

Thereby, the objective of policy gradient transforms from the policy value $J(\pi)$ to the policy improvement $J(\pi)-J(\pi_t)$ according to the Lemma \ref{lemmaLB}. The first term on the right-hand side of the inequality is often referenced as the surrogate objective, and the second term serves for an extra refinement. 
Following policy improvement lower bound, PPO \cite{PPO} clips the probability ratio as the penalty for excessive policy updates and then ignores the second term:
\begin{equation}
    L^{\mathrm{clip}}(\theta)=\hat{\mathbb{E}}_t[\min{(r_t(\theta)\hat{A}_t,\mathrm{clip}(r_t(\theta),1-\epsilon,1+\epsilon)\hat{A}_t)}]
\end{equation}
where $\mathrm{clip}(x,l,b)=\min(\max(x,l),b)$, and the probability ratio $r_t(\theta)={\pi^\theta(a_t|s_t)}/{\pi_t^\theta(a_t|s_t)}$ is calculated using samples generated by the current policy $\pi_t$.

In practice, since the GAE is calculated outside the optimization loop, the advantage value $A$ does not contribute to gradients of the objective function. Consequently, the actual policy gradient relies primarily on the probability ratio term $r$. 

\subsection{Off-Policy Gradient}
The off-policy gradient can be written as \cite{offPAC}:  

\begin{equation}\begin{aligned} \nabla J (\pi) &=
    \nabla \left[ \sum_{s\in S}d^b(s)\sum_{a\in A} \pi(a|s)Q_\pi (s,a) \right]\\
    &\approx \sum_{s\in S}d^b(s)\sum_{a\in A} \nabla \pi(a|s)Q_\pi (s,a) \\
    &=\mathbb{E}_{(s,a)\sim b}[r_b(s,a)\nabla \log{\pi(a|s)}Q_\pi (s,a)]
\end{aligned}\end{equation}
where $r_b(s,a)={\pi(a|s)}/{b(a|s)}$, the $d^b(s)$ is the state distribution of eligibility traces generated by prior policies $b$.

Here, the probability ratio term $r$ highlights the crucial distinction between on-policy and off-policy approaches: in on-policy learning, the reference policy, i.e., the denominator, coincides with the current policy $\pi_t$, whereas in off-policy learning, it refers to previous policy $b$.

\subsubsection{Deterministic Policy Gradient}
Unlike stochastic policy gradient algorithms proceed by sampling from a parametric probability distribution $\pi_\theta(a|s)=\mathbb{P}[a|s;\theta]$, deterministic policies eliminate the variance for a direct gradient of action-value function. Consider a stochastic policy $\pi_{\mu_{\theta},\sigma}$ such that $\pi_{\mu_{\theta},\sigma}(a|s)=\nu_{\sigma}(\mu_{\theta}(s),a)$, where $\sigma$ is a parameter controlling the variance. Then, the deterministic policy $\mu$ can be obtained by setting the $\sigma \to 0$ of stochastic policy $a = \mu_{\theta}(s)$. And the gradient of deterministic policy can be written as\cite{DPG}:
\begin{equation}
    \nabla_{\theta} J (\mu_{\theta})=\mathbb{E}_{s\sim\rho^\mu}[\nabla_{\theta} \mu_{\theta}(s) \nabla_{a} Q^{\mu}(s,a)].
\end{equation}

Thus, methods based on deterministic policy gradient, including DDPG\cite{DDPG}, TD3\cite{TD3} and SAC\cite{SAC}, are predominantly applied for tasks with continuous action spaces. These policies generate continuous real-valued outputs, which are then passed to the Q-function to yield estimations for policy gradients. That is, the policy gradients of these algorithms are derived from the Q-function, rather than directly from the action probability distribution $\pi$. To ensure adequate exploration, they utilize trajectories generated by distinct behavior policies, which makes these DPG style algorithms \textit{off-policy}.

\section{Extended Off-policy Proximal Policy Optimization Algorithm}

In this work, we introduce an off-policy improvement lower bound that relates the target policy to the expectation of any reference policies, called Extended Off-Policy Improvement Lower Bound. This lower bound enables a particular off-policy sampling pattern that leverages the last $M$ policies. Then, we extend the clipping mechanism to better handle with off-policy data correspondingly. 

\subsection{Extended Off-Policy Improvement Lower Bound}

In this section, we present an expectation form of policy improvement lower bound based on the ff-Policy Improvement Lower Bound \cite{GePPO}. Accounting for distributions shift between prior policies, the Generalized Policy Improvement Lower Bound theorem supports sample reuse \cite{GePPO}, but it is still an on-policy improvement lower bound. Then, to further reduce the repeated calculations of the advantage function, we relax the conservative constraint to be a difference between training future policy and expectation of prior policies, which means an \textit{off-policy} improvement lower bound.  

First, we derive an inequality between training future policy $\pi$ and any reference policy $\pi_{\mathrm{ref}}$ from the on-policy improvement lower bound Lemma \ref{lemmaLB}.

\begin{lemma}[Policy Improvement Lower Bound Between Any Reference Policy]\label{lemma:SB}
    Consider a training future policy $\pi$ and any reference policy $\pi_{\mathrm{ref}}$, there is an inequality between them:
    \begin{equation}
        \begin{aligned}
            J(\pi)-J(\pi_{\mathrm{ref}}) \ge & \frac{1}{1-\gamma} \mathop{\mathbb{E}}\limits_{(s,a)\sim \pi_{\mathrm{ref}}} \left[ \frac{\pi(a|s)}{\pi_{\mathrm{ref}}(a|s)} A^{\pi_{\mathrm{ref}}}(s,a) \right]\\
            &- \frac{2\gamma C^{\pi,\pi_{\mathrm{ref}}}}{(1-\gamma)^2}\mathop{\mathbb{E}}\limits_{(s,a)\sim \pi_{\mathrm{ref}}}\Big[\mathrm{TV}(\pi,\pi_{\mathrm{ref}})(s)\Big]
        \end{aligned}
    \end{equation}
\end{lemma}

\begin{proof}
    This inequality is derived from the disparity between any policy $\pi$ and a reference policy $\pi_\mathrm{ref}$ with the same starting state distribution \cite{Approx}. And the divergence between the state visitation distributions is bounded by the divergence of the policies \cite{CPO}. See the Appendix \ref{Approof} for details. 
\end{proof}

Then, the reference policy $\pi_{\mathrm{ref}}$ could come into prior policies $\pi_{t-i}$, $i=0,1,...,M-1$, when sampling last $M$ policies from replay buffer. Thus, we apply Lemma \ref{lemma:SB} $M$ times with the sampling probability distribution $\nu=[\nu_0 ,..., \nu_{M-1}]$, then get a mathematical expectation form of lower bound between training policy $\pi$ and prior policy $\pi_{t-i}$.

\begin{theorem}[Extended Off-Policy Improvement Lower Bound]\label{Theo1}
Consider prior policies $\pi_{t-i}$, $i=0,1,...,M-1$, where $\pi_t$ represents the current policy. Sampling from replay buffer according to random probability distribution $\nu=[\nu_0 ,..., \nu_{M-1}]$ over the prior $M$ policies and for any training future policy $\pi$, we have: 

\begin{equation}\begin{aligned}
    J(\pi)- &\mathop{\mathbb{E}}\limits_{i\sim\nu} [J(\pi_{t-i})] \ge \frac{1}{1-\gamma}\mathop{\mathbb{E}}\limits_{i\sim\nu}\Big[ \mathop{\mathbb{E}}\limits_{(s,a)\sim \pi_{t-i}} \big[ r_{t-i}(s,a) A^{\pi_{t-i}}(s,a) \big] \Big]  \\
    & - \mathop{\mathbb{E}}\limits_{i\sim\nu}\left[\frac{2\gamma C^{\pi,\pi_{t-i}}}{(1-\gamma)^2} \mathop{\mathbb{E}}\limits_{(s,a)\sim \pi_{t-i}}\big[\mathrm{TV}(\pi,\pi_{t-i})(s) \big]\right]
\end{aligned}\end{equation}
where $r_{t-i}(s,a)={\pi(a|s)}/{\pi_{t-i}(a|s)}$, $C^{\pi,\pi_{t-i}}$ and $\mathrm{TV}(\pi,\pi_{t-i})(s)$ are defined as in Lemma \ref{lemmaLB}.
\end{theorem}

\begin{proof}
    As Lemma \ref{lemma:SB} represents a lower bound between training policy $\pi$ and any reference policy $\pi_{\mathrm{ref}}$, we apply Lemma \ref{lemma:SB} $M$ times in terms of convex combination of sampling probability distribution $\nu=[\nu_0 ,..., \nu_{M-1}]$ for the reference policies $\pi_{t-i}$, $i=0,1,...,M-1$. Then calculating the expectation with respect to distribution $\nu$ on both sides of the inequality, we can get the theorem \ref{Theo1}. See the Appendix \ref{Approoftheo} for details. 
\end{proof}

Theorem \ref{Theo1} expands the policy improvement lower bound to be between future policy $\pi$ and the expectation of prior policies $\pi_{t-i}$, $i=0,1,...,M-1$ in the replay buffer. Thereby, Theorem \ref{Theo1} provides theoretical support for an off-policy training pattern containing $M$ prior policies, specifically the Extended Off-policy Pattern. This implies that policy improvement throughout learning iterations could be guaranteed by maximizing the lower bound. Then, we can refer to the Theorem \ref{Theo1} as the basic form of surrogate objective in the subsection \ref{sec:surrogate}. And the probability distribution $\nu$ describes how samples are selected from the buffer, which we set to be a uniform distribution for random sampling by default. 

\subsection{Sample Analysis}\label{sec:sample}

\begin{figure*}[th]
    \centering
    \includegraphics[width=\textwidth]{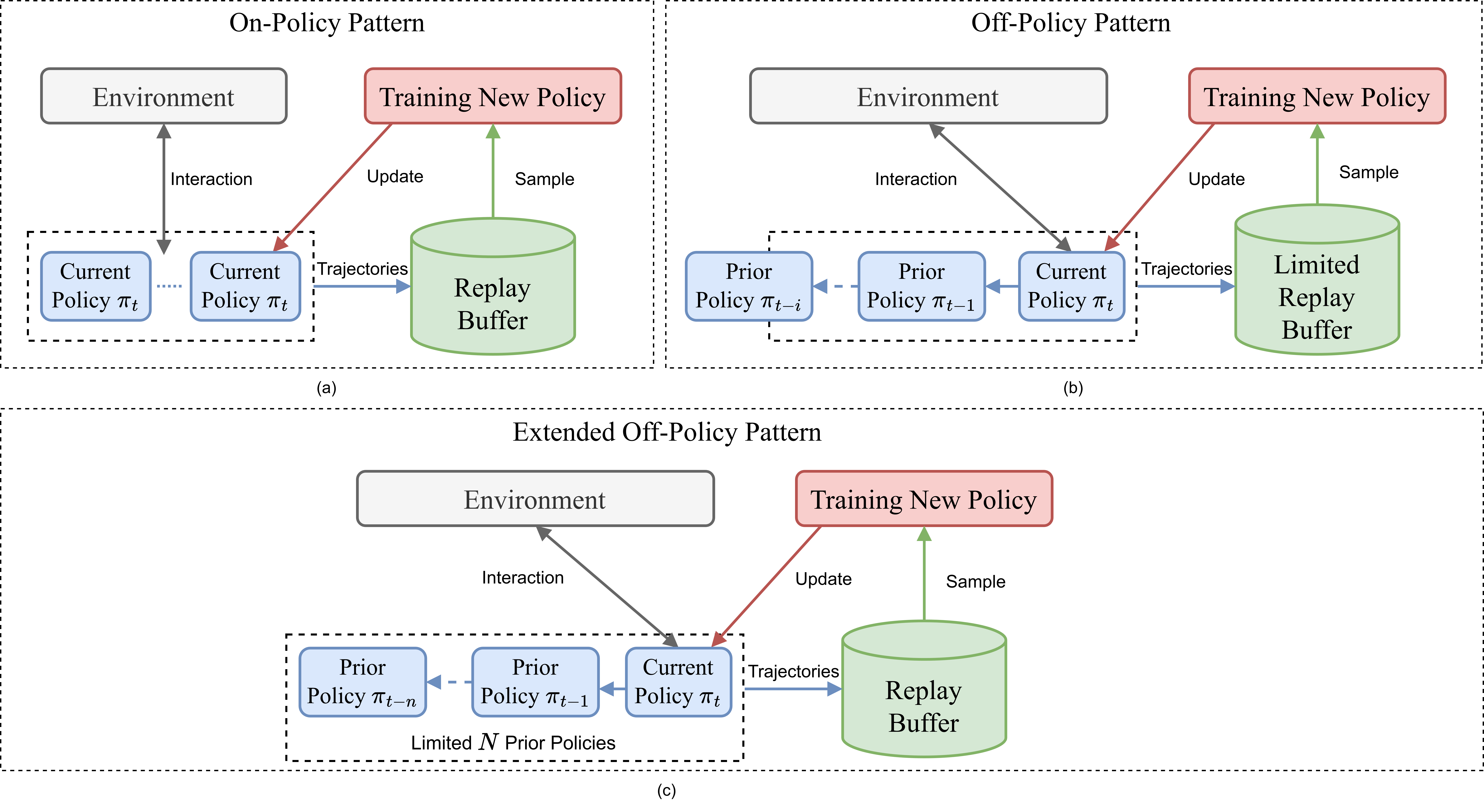}
    \caption{Intuitive comparison between DRL sampling-and-training patterns. \textbf{(a):} On-policy pattern samples from the current policy, which usually requires $N$ parallel online interactions to generate enough samples for training. \textbf{(b):} Off-policy pattern separately keeps each action-state transition from distinct periods in the limited-size queue-like replay buffer. Reusing past pieces improves sample efficiency, but the shift between policies becomes a thorny issue. \textbf{(c):} Different from limiting the number of individual entries in the replay buffer, our Extended Off-policy sampling pattern takes the trajectories by a certain policy as a cohesive unit. It could improve the efficiency while constrict the distribution shift by finite $M$ prior policies.
    }
    \label{fig:ppatt}
  \end{figure*}

Unlike most of off-policy algorithms which storing every immediate transition in a large-sized replay buffer, Theorem \ref{Theo1} suggests that only episodes from the latest $M$ policies are retained. This Theorem presents an enhancement of the on-policy pattern by extending the current policy to prior $M$ policies from the replay buffer. As shown in Figure \ref{fig:ppatt} for a visual illustration, we call this proposed training pattern as \textit{Extended Off-policy} pattern. 

In more detail, assume that $N_{\mathrm{on}}$ parallel agents are involved in one on-policy training episode. The extended off-policy employs $N_{\mathrm{off}}$ agents while maintaining $M$ prior policies in the replay buffer, where $N_{\mathrm{off}} = {N_{\mathrm{on}}}/{M}$ ensures every entry to be trained the same times. In other words, the extended off-policy pattern requires only $1/M$ of the task samples to perform the same amount of training updates as on-policy methods. This results in higher scores in evaluation tasks with the same number of interaction steps. On the other side, the \text{completely} off-policy pattern, like DDPGs, may introduce increased deviation which often makes training unstable, especially in discrete tasks.

In comparison to GePPO \cite{GePPO}, Theorem \ref{Theo1} establishes a broader expectation of a generalized policy improvement lower bound. Specifically, the reference policy for the advantage function term $A^{\pi_{\mathrm{ref}}}(s,a)$ in Theorem \ref{Theo1} is the prior policy $\pi_{t-i}$, whereas the \textit{Generalized Policy Improvement Lower Bound} refers to the current policy $\pi_t$. This distinction has a significant implication for computational efficiency: our pattern calculates the advantage function only once at the end of each online episode and reuses it from the replay buffer, while GePPO recalculates advantages for every trajectory before each training iteration. That is, when maintaining the same number of $M$, GePPO recalculates advantages $M-1$ more times than ExO-PPO.

The selection of parameter $M$ influences the amount of off-policy data included in training, with larger $M$ incorporating more past policies. When acquiring new interaction data becomes challenging, setting a larger $M$ allows for more completed training on the available data, but at cost of increasing computational complexity. Therefore, in online training tasks, we typically choose a moderate $M$ in a range of $3$ to $6$, to balance sampling efficiency and computational stability. {In the empirical evaluation,} we searched a suitable $M = 4$ that balances efficiency and stability, and the results are presented on the right side of Figure \ref{fig:paras}. In addition, the standard on-policy improvement lower bound in Lemma \ref{lemmaLB} can be recovered by setting $M = 1$, where the replay buffer always fills with transitions from the current policy $\pi_t$. 

\subsection{Extended Ratio Objective with Exponentially Decaying Edge}\label{sec:surrogate}

As mentioned in Section \ref{sec:pre:PILB}, the first term on the right-hand side of Theorem \ref{Theo1} is close to what the policy gradient optimize for, while the second term often acts as a penalty term. The former constitutes the optimization objective function, and the penalty term restrains the expected total variation distances between the training future policy $\pi$ and the last $M$ policies. By integrating the characteristics of both terms, various algorithms propose their surrogate objectives. In this section, we introduce our proposed ExO-PPO objective, the extended ratio objective with exponentially decaying edge.    

\begin{figure}[ht]
    \centering
    \includegraphics[width=\columnwidth]{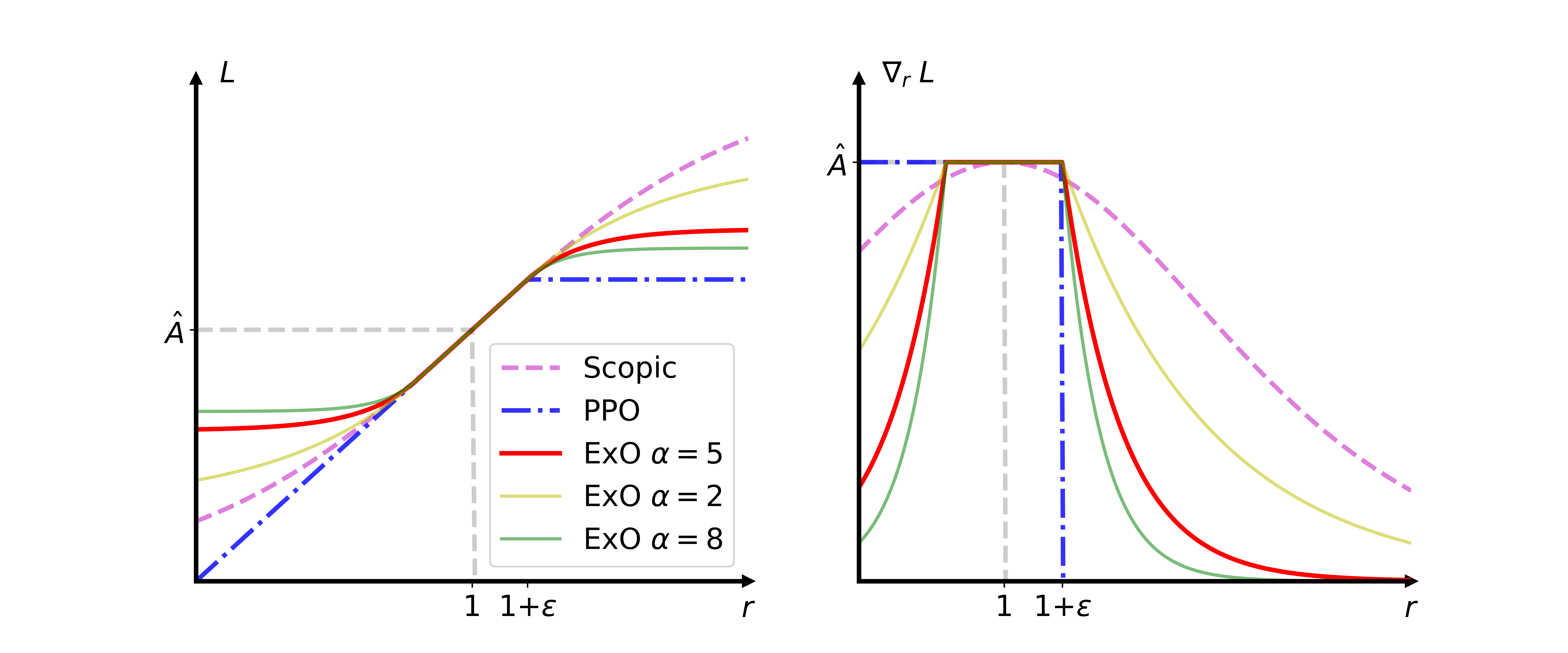} 
    \caption{Left: The optimization objective functions $L$ of stochastic policy gradient versus the probability ratio $r={\pi}/{\pi_{t-i}}$ when $\hat{A}>0$. Right: The corresponding gradients of these objective functions $\nabla_r L$ with respect to $r$ are shown. Compared with the original PPO objective (blue dashdot line) and the Scopic objective \cite{TheSuffi} with $\tau = 2$ (purple dashed line), our proposed ExO-PPO objectives with different $\alpha$ (a cluster of solid lines) are smooth curves with continuous first order derivatives. And the center symmetry point $(1,\hat{A})$ is the {\textit{on-policy point}} of optimization.}
    \label{figfunc}
\end{figure}

Since the policy gradient relies on the probability ratio term $r$ primarily, the clipping mechanism in PPO restricts the total variation between policies by truncating parts that exceed the clipping range $[1 -\varepsilon,1+\varepsilon]$.  However, the practical probability ratio may not be well constrained by the clipping range \cite{TrulyPPO,acloserl}, which could become more obvious in the off-policy setting. Off-policy data means the importance sampling ratio ${\pi}/{\pi_{t-i}}$ would deviate more sharply from the on-policy point $r=1$, whose gradient might be easily discarded by the clipping mechanism. Thus, to involve lager gradient space while properly attenuating the weights of remote parts for constraint, we propose the extended ratio objective function, a modified surrogate objective according to the extended off-Policy improvement lower bound in Theorem \ref{Theo1}: 

\begin{equation}\begin{aligned}\label{eq:ExO}
    L^{ExO}(\theta) = 
    \mathop{\mathbb{E}}\limits_{i\sim\nu}\Big[\mathop{\mathbb{E}}\limits_{(s,a)\sim \pi_{t-i}} \big[ \xi_{t-i}(\theta) & A^{\pi_{t-i}}(s,a) \big]  - \beta \mathrm{KL}(\pi,\pi_{t-i}) \Big]  
\end{aligned}\end{equation} 
where $\xi_{t-i}(\theta)$ is the segmentation function that extends the clipped ratio $r_{t-i}$ by a parameterized exponential function, and we also adopt KL divergence for the penalty term. 
Then, we define the extended ratio with exponentially decaying edge function $\xi$:

\begin{definition}[Extended Ratio with Exponentially Decaying Edge]\label{def1}
    Consider the training future policy $\pi$, the reference prior policy $\pi_{t-i}$, the probability ratio $r = \frac{\pi}{\pi_{t-i}}$, and clipping parameter $\varepsilon$. 
    Our extended ratio function keeps linear within clipping range and exponentially decaying outside clipping range, whose segmentation function $\xi_{t-i}(s,a)$ is defined as:

    \begin{equation}
        \xi_{t-i}(s,a)=\begin{cases}
        (1-\varepsilon)- (1 - \varphi^{+}_{t-i})/\alpha, &r_{t-i} < 1-\varepsilon\\ 
        r_{t-i}(s,a), & otherwise \\
        (1+\varepsilon)+ (1 - \varphi^{-}_{t-i})/\alpha, &r_{t-i} \ge 1+\varepsilon
    \end{cases} 
    \end{equation}
    where $r_{t-i}(s,a)$ is defined as in Theorem \ref{Theo1}, $\varphi^{\pm}_{t-i} = \exp\big[\alpha\times(\varepsilon \pm(r_{t-i}(s,a) - 1)) \big]$, and $\alpha$ is a parameter controlling the slope of the extended ratio. 
\end{definition}

The gradient of extended ratio function $\xi_t$ with respect to $r_t(\theta)$ is:
\begin{equation}
\frac{\partial\xi_t(s,a)}{\partial r_t(s,a)}=\begin{cases}
    \varphi^{+}_t, &r_t < 1-\varepsilon \\ 1, & (1-\varepsilon)\le r_t < (1+\varepsilon) \\
    \varphi^{-}_t, &r_t \ge 1+\varepsilon
\end{cases} \end{equation} 
where we set the same default clipping parameter $\varepsilon=0.2$, as in the clipping mechanism of PPO.

Our extended ratio function is a smooth curve with a continuous first order derivative, symmetric about the center of {on-policy} point $(1,\hat{A})$. A cluster of the objective and their gradient function are drawn in Figure \ref{figfunc} with different $\alpha$. This plot illustrates that a larger value of $\alpha$ could accelerate the decay of the gradient outside the clipping range, which is closer to the clipping mechanism. On the contrary, the opposite case offer a larger gradient space. For instance, when $r=2$, we obtain the following values of $\xi(r)$ and its gradient listed in Table \ref{table:xivalues}. In practice, we set the default value $\alpha=5$ for most tasks, which achieves relatively better results in empirical experiments. 

\begin{table}[h]
    \centering
    \begin{tabular}{ccccc} %
    \toprule
    Value of $\alpha$ & $\alpha=2$ & $\alpha=5$ & $\alpha=8$ & clipping mechanism ($\varepsilon=0.2$) \\
    \midrule
    $\xi(r=2)$              &1.60    &1.40   & 1.32  &1.2\\
    $\nabla_{r}\xi(r=2)$    &0.20    &0.018  &0.0016 &0\\
    \bottomrule
    \end{tabular}
    \caption{The value of $\xi$ and its gradient vary with different values of $\alpha$ when $r=2$ and $\varepsilon=0.2$.}
    \label{table:xivalues}
\end{table}

Figure \ref{figfunc} also compares the objective functions and their gradients between the clipping objective, the Scopic objective \cite{TheSuffi}, and our extended objective. Our curve takes a wider range of variance in ratio than the clipping objective, and better suppresses the excessive large gradients than the Scopic objective.


Since the extended ratio function relaxes the policy constraint, we also adopt KL divergence as the penalty term to further ensure training stability, as many other methods do \cite{TRPO,TheSuffi,ESPPO}. Constrained policy optimization \cite{CPO} declares that TV-divergence and KL-divergence are related by: 
\begin{equation}
    \mathop{\mathbb{E}}\limits_{s \sim d^{\pi}}\big[D_\mathrm{TV}(\pi,\pi')(s) \big] \to \sqrt{\frac{1}{2}\mathop{\mathbb{E}}\limits_{s \sim d^{\pi}}\big [D_\mathrm{KL}(\pi,\pi')(s) \big]}.
\end{equation}

Then, our complete algorithm is detailed in the following subsection.

\subsection{Algorithm} 

Like other actor-critic pattern deep reinforcement algorithms, ExO-PPO has two networks: a policy network and a value network. These two are both deep neural networks. The surrogate objective of policy network is the extended ratio objective, given by $L^{ExO}$ in Equation (\ref{eq:ExO}). Training entries are sampled from a replay buffer following the extended off-policy sampling pattern. Critic network approximates the advantage values of states $\hat{A}^\pi$ using generalized advantage estimation \cite{GAE}, which is calculated closely analogous to TD($\lambda$) with lower variance and bias. During training, the learning rate would gradually decay for better convergence. Algorithm \ref{algor1} presents the pseudocode of ExO-PPO. 

\begin{algorithm}[tb]
    \caption{Extended Off-policy Proximal Policy Optimization Algorithm (ExO-PPO)}
    \label{algor1}
    {Initialize the replay buffer $D$, whose capacity is the number of prior policies $M$ times the number of episodic steps $T$}
    \begin{algorithmic} 
    \For {episode = 1, $N$}
        \For {t = 1, $T$}
        \State Collect $n$ transitions $(s_t,a_t,r_t,s_{t+1})$ with {current} $\pi_t$,
        \State Compute advantage estimates $\hat{A}_t$, 
        \State Store transitions $(s_t,a_t,r_t,s_{t+1},\pi_t(s_t),\hat{A}_t)$ in replay buffer $D$.
        \EndFor
    \State {Randomly} sample batches of transitions $(s_j,a_j,r_j,s_{j+1},\pi_j(s_j),\hat{A}_j)$ from $D$ {containing $M$ prior policies},
    \State Approximately optimize the {Extended Off-policy} objective $L^{ExO}(\theta)$ through minibatch stochastic gradient ascent.
    \EndFor
    \end{algorithmic}
\end{algorithm}

\section{Experiments} 


In this section, we analyze the performance of ExO-PPO algorithm on tasks in Gymnasium \cite{gym,gymnasium} and MuJoCo \cite{mujoco} environments. To speed up data collecting, we adopted the latest EnvPool 0.8.4 \cite{envpool} as our training simulator, which is a batched environment pool with Pybind and thread pool. Specifically, we consider several single-agent Atari games environments, including Pong-v5, Breakout-v5, etc. The observation space on these environments are stacked gray images and the action policy is under Gibbs distribution. And continuous control tasks Ant-v4, HalfCheetah-v4 and Walker2d-v4 from MuJoCo are learned by parameterized policies following Gaussian policy. 

We compare ExO-PPO algorithm with a few novel PPO variants ESPPO \cite{ESPPO}, P3O-Scopic \cite{TheSuffi}, and the original PPO algorithm as the baseline. All of these algorithms are applied with the same network structure and learning rates, and tested in varied environments with four random seeds. In detail, for these on-policy algorithms, we set the default parallel environment number to $8$, episodic run steps to $256$ per environment, and the training batch size to 256. And, the default clipping parameter for PPO is $\varepsilon=0.2$, temperature $\tau = 2$ for P3O-Scopic, and the stopping threshold $\delta = 0.25$ for ESPPO. 
For the off-policy algorithms, we evaluated the SAC \cite{SAC,dissac} and TD3 \cite{TD3} (in continuous control tasks) with the same network structure, replay buffer size, and learning rates as on-policy algorithms. For mixed off-policy algorithms such as GePPO \cite{GePPO}, off-policy PPO \cite{Off-PPO} in (analyzed in Section \ref{abla}), and ExO-PPO, we select the number of prior policies $M=4$ and parallel environment number to $2$, ensuring that each sample is calculated the same times as on-policy algorithms. Regarding the edge decaying factor in ExO-PPO, we set it is $\alpha=5$ by default. We empirically investigate this parameter along with $M$ in the section \ref{sec:para}. Additionally, since no data augmentation is applied in the replay buffer, the corresponding policy weights $\nu$ are a uniformly distributed.



\begin{figure*}[ht]
    \centering
    \includegraphics[width=\textwidth]{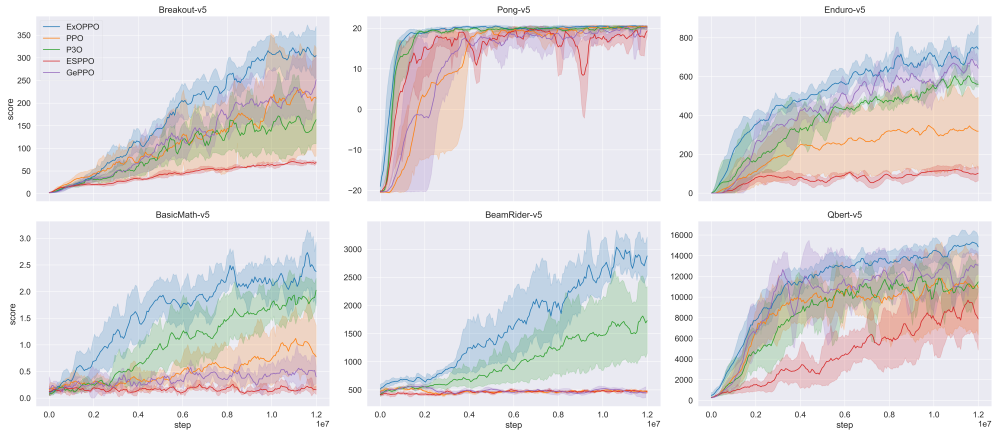} 
    \caption{Average performance throughout training across Atari games. Shading denotes the range of multiple runs and solid lines are the mean return of each algorithm. ExO-PPO (the blue line) outperforms in the most of the tasks in both early and final stages.}
    \label{per1}
\end{figure*}

\begin{figure*}[ht]
    \centering
    \includegraphics[width=\textwidth]{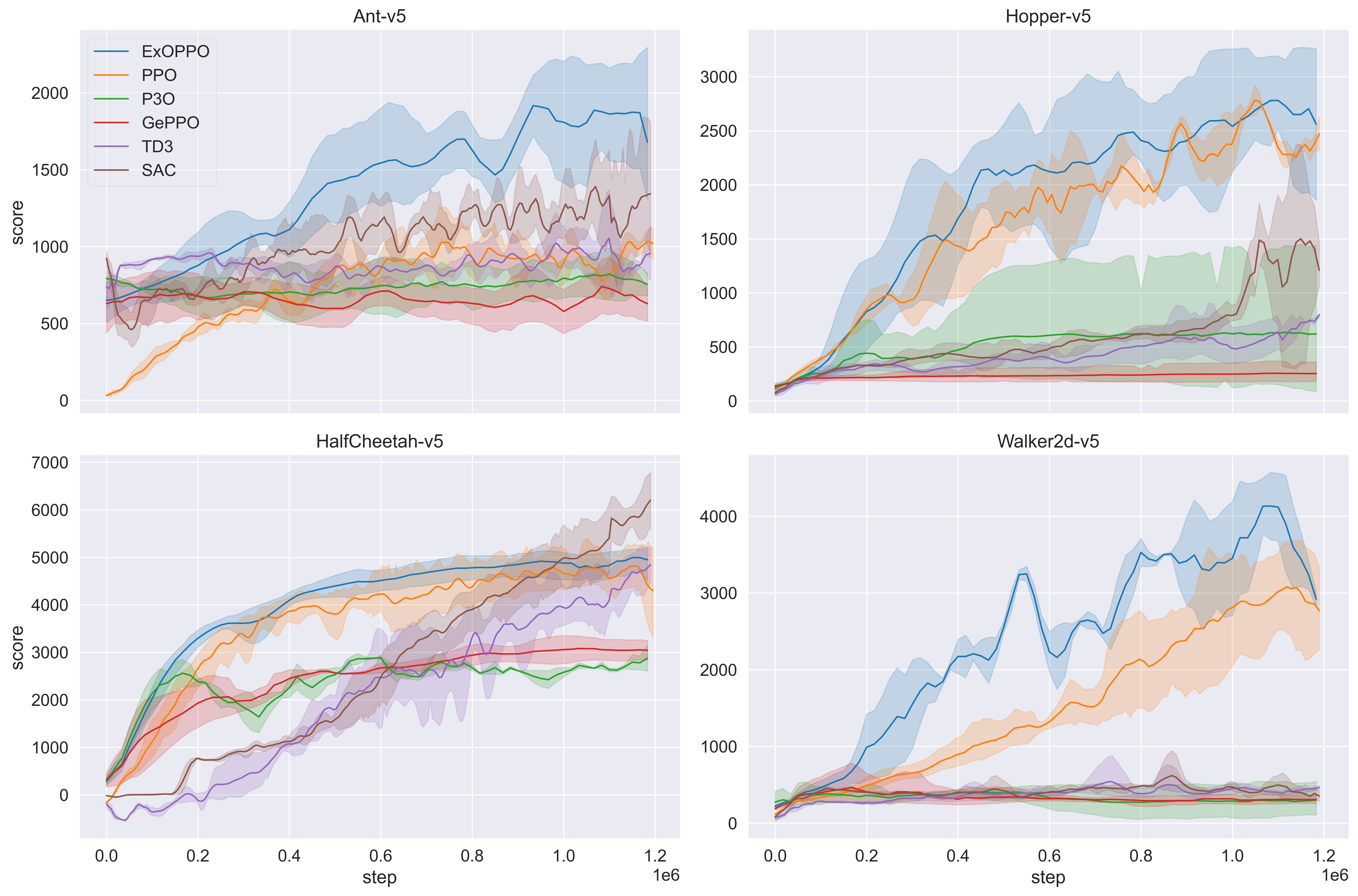}
    \caption{Average performance throughout training across continuous MuJoCo tasks.}
    \label{fig:mujovalue}
\end{figure*} 

\begin{figure*}[ht]
    \centering
    \includegraphics[width=\textwidth]{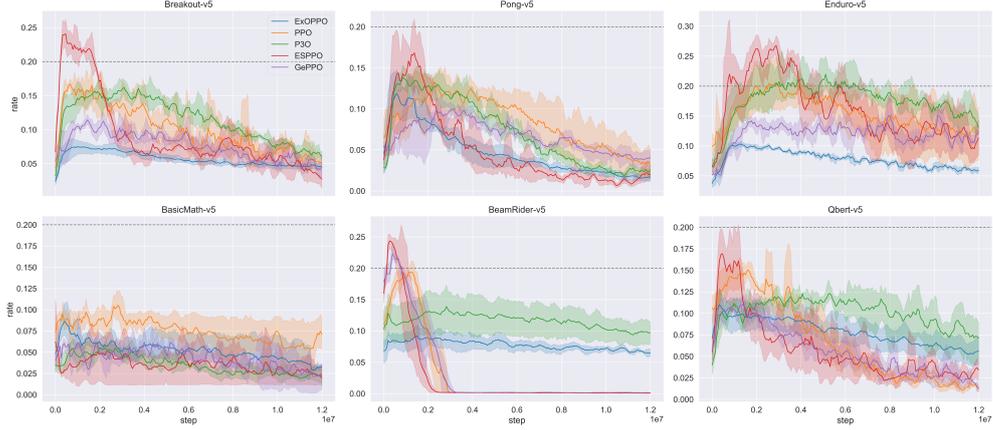}
    \caption{Absolute logarithm of probability ratio throughout training across Atari games. Red vertical dashed lines represent the absolute logarithm of the probability ratio $y=0.2$, which is an indicator of whether the surrogate objective restrains the disparity between training updates.}
    \label{rate1}
\end{figure*}
\subsection{Scoring Performance} We first evaluated several common Atari games for discrete tasks over 12M interaction steps, using the test episode scores after multiple training iterations as statistical metric. As illustrated in Figure \ref{per1}, ExO-PPO demonstrates fast and reliable learning across most tasks, resulting in improved performance in test episodes across most environments in both early and final stages. In the Pong environment, which is simple but unique due to a score ceiling, ExO-PPO reaches the maximum score with minimal training steps and exhibits relatively low performance fluctuations. 

For the continuous control tasks, we also compared ExO-PPO with other representative off-policy algorithms, including TD3 \cite{TD3} and SAC \cite{SAC}.
Owing to the properties of stochastic policy gradient, Gaussian policies exhibit high sensitivity to their initial mean, variance, and rates of change for these parameters. As demonstrated in Figure \ref{fig:mujovalue}, with proper hyperparameter tuning, ExO-PPO can also achieve competitive scores on these tasks. 


\subsection{Stability and Training Efficiency} 
To grasp how the probability ratio ${\pi}/{\pi_\mathrm{ref}}$ changes and affects the stability of training during training, we calculate the expectation of its absolute logarithm, defined as $y = \mathbb{E} |\ln{\pi} - \ln{\pi_\mathrm{ref}}|$. Since the probability ratio provides the primary gradient for policy update, its value significantly influences and reflects the stability of training. As the clipping parameter is set $\varepsilon = 0.2$, our measurements yield $y(1.2)=0.182$ and $y(0.8)=0.223$. Based on that, we consider $y=0.2$ as an indicator to assess whether the surrogate objective properly restrains the probability ratio.

Figure \ref{rate1} illustrates the results of this measure. Together with Figure \ref{per1}, the training performance would be adversely affected if the ratio is too large or too small, with ESPPO serving as an example. Since ESPPO does not cut the probability ratio until it reaches some threshold, its curves in Figure \ref{rate1} exhibit the largest variation scope while achieving relatively the lowest scores in Figure \ref{per1}. In contrast, the probability ratio of ExO-PPO changes smoothly throughout the training process, which indicates that the extended ratio objective effectively limits variation for stable training even in an off-policy setting with higher variance. 


\begin{figure}[ht]
    \centering
    \includegraphics[width=\columnwidth]{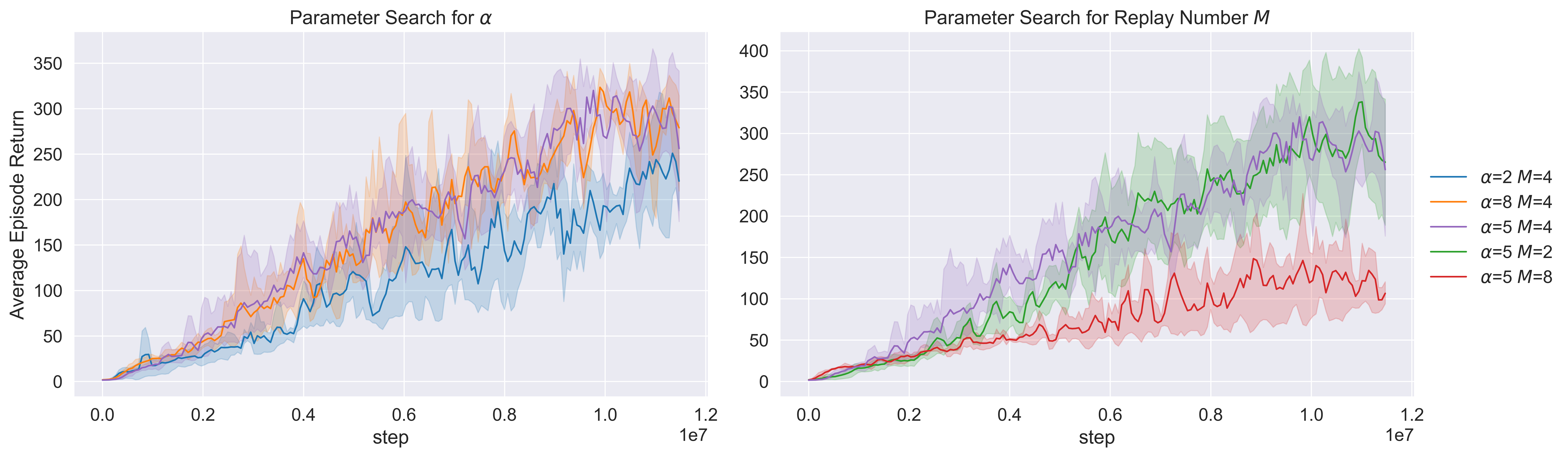}
    \caption{Comparative analysis of two key hyperparameters: the exponential decay rate of the extended edge $\alpha$ (left) and the number of prior policies $M$ in the replay buffer (right). The purple line illustrates superior performance by the default hyperparameters: $\alpha=5$ and $M=4$.}
    \label{fig:paras}
\end{figure}

\subsection{Parameter Search}\label{sec:para} The exponential edge control parameter $\alpha$ and the number of prior policies $M$ in the replay buffer are two special hyperparameters in ExO-PPO algorithm that are explored in this section. The $\alpha$ are defined in the Definition \ref{def1} that controls the attenuation scope of the gradient. Objective curves with different $\alpha$ {are drawn} in Figure \ref{figfunc} and {their performances are compared} on the left side of Figure \ref{fig:paras}. {As mentioned before, a smaller $\alpha$ would result in less aggressive gradient attenuation away from the on-policy point, enabling greater exploration. Conversely, a larger $\alpha$ brings the algorithm closer to PPO. Referring to the performance comparison, a moderate $\alpha$ could be chosen as the default value.}

The effect of $M$ is discussed in the section \ref{sec:sample} and figured on the right side of Figure \ref{fig:paras}. A larger value of $M$ enhance the utilization of off-policy samples but also widens the distribution shift. Conversely, a smaller $M$ is considered to increase the training stability. Our experimental results demonstrate that a moderate value of $M$ could achieve relatively higher scores. In this study, we select $\alpha=5$ and $M=4$ as default hyperparameters, balancing training stability and the sample efficiency.

\subsection{Ablation Study}\label{abla} 
In this subsection, we analyze the impacts of off-policy improvement lower bound and the extended ratio objective relatively on our algorithm. Related results are displayed in the Figure \ref{fig:alba},\ref{fig:rat2}.

\begin{figure}[ht]
    \centering
    \includegraphics[width=\columnwidth]{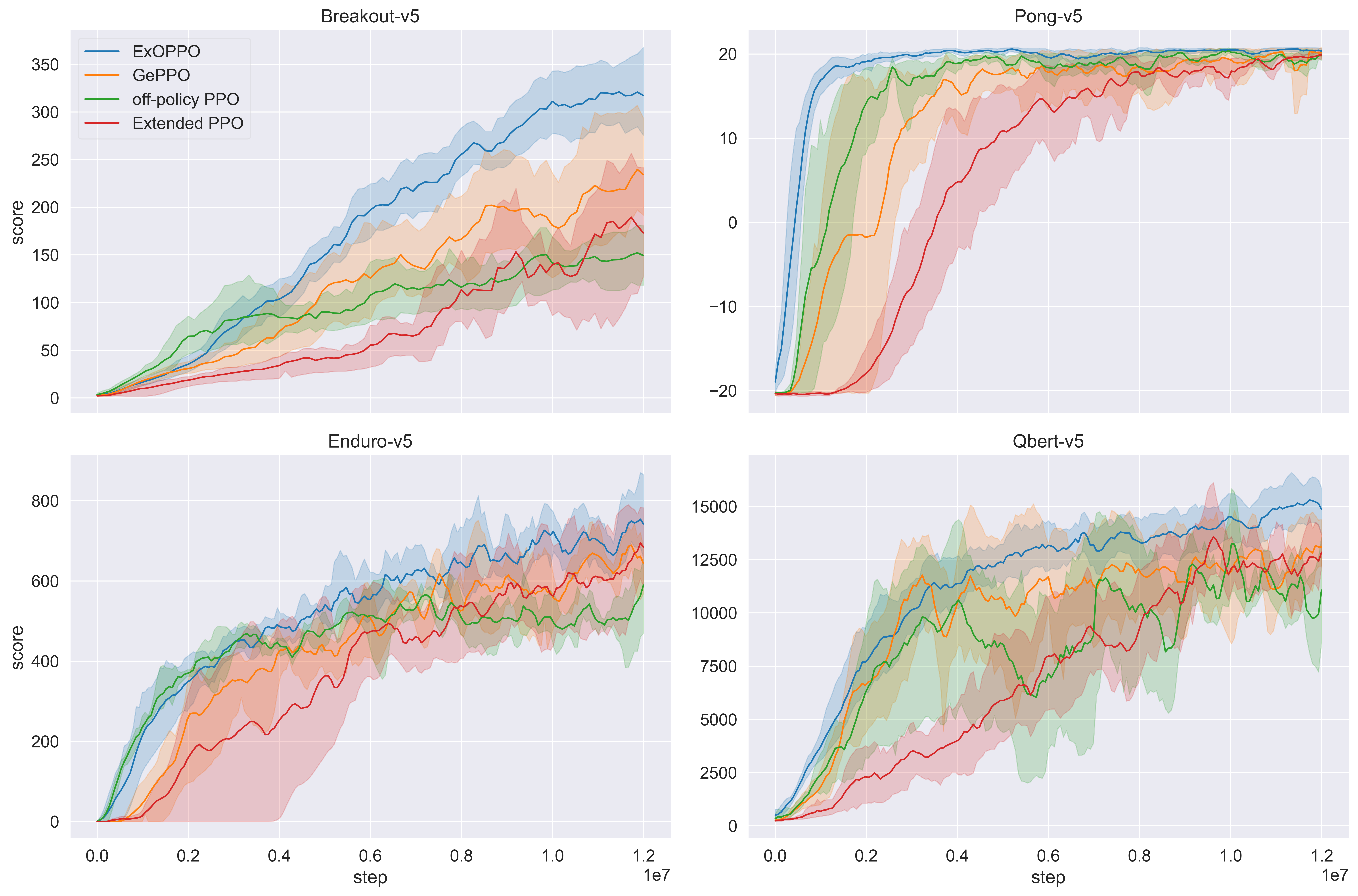}
    \caption{Performance throughout training across Atari games for ablation study.}
    \label{fig:alba}
\end{figure}

\begin{figure*}[ht]
    \centering
    \includegraphics[width=\columnwidth]{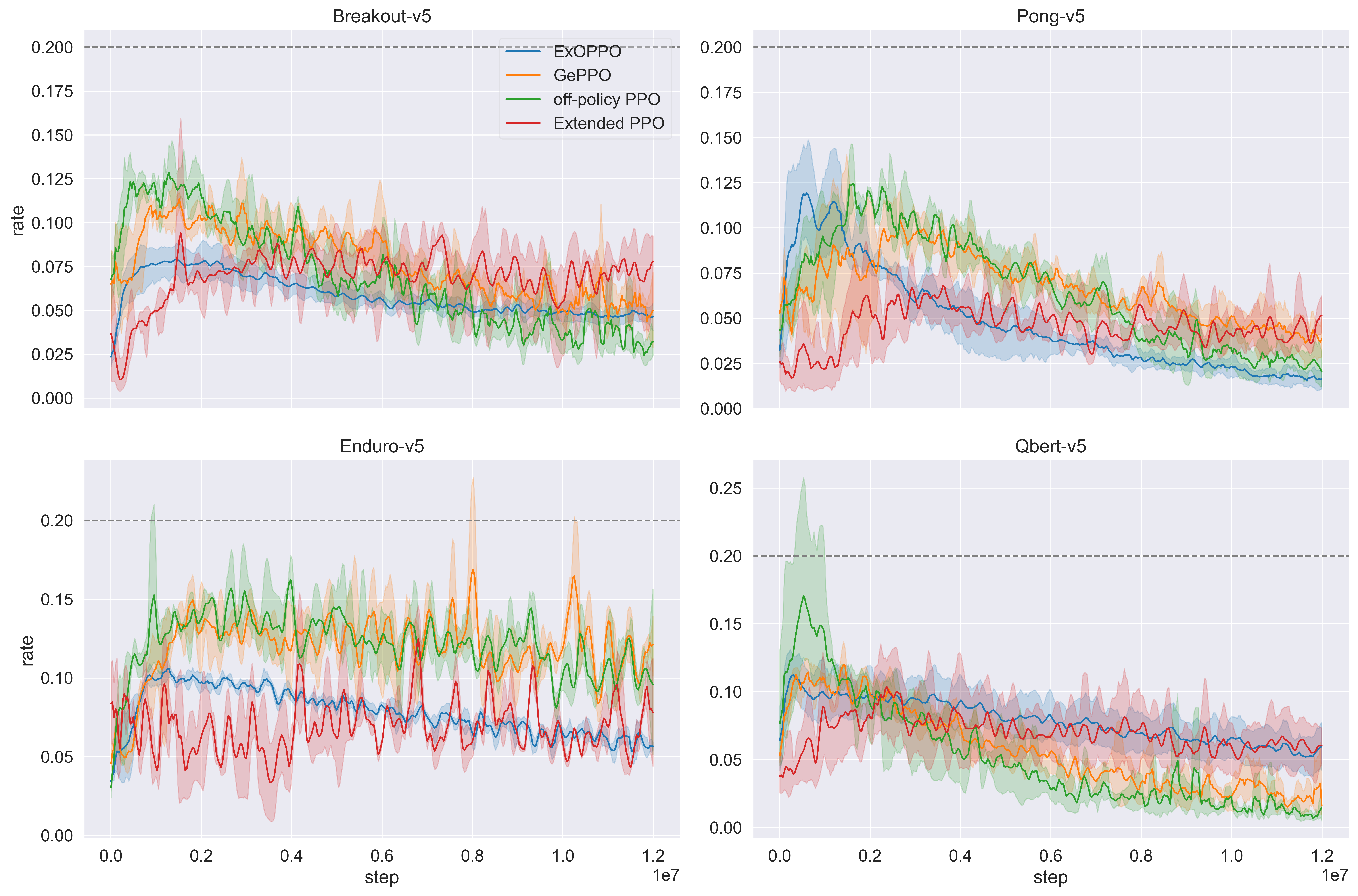}
    \caption{Absolute Logarithm of probability ratio throughout training across Atari games for ablation study.}
    \label{fig:rat2}
\end{figure*}

\subsubsection{Extended Off-policy Improvement Lower Bound}\label{albstu}

This subsection evaluates the effectiveness of the extended off-policy improvement lower bound by comparing with \textit{Extended PPO}, an on-policy algorithm with our extended ratio objective additionally. During training and testing, the Extended PPO and the ExO-PPO share the same edge parameter $\alpha=5$ and other training hyperparameters. As scoring results presented in the Figure \ref{fig:alba}, Extended PPO showed slower learning improvements than ExO-PPO in various tasks. This is because that when the ratio deviates from the center during training, its gradients are retained and compressed by the expanded edge rather than being truncated. Consequently, the batch gradients are reduced through averaging, instead of being a fixed value of $1$ as in clipped ratio, leading to decreased training efficiency. This trend can also be observed in the Figure \ref{fig:rat2}, where the overall change of training policy is relatively small in the early stages. Furthermore, EXO-PPO also outperforms both off-PPO and GePPO on most of these tasks. These results suggesting that incorporating the extended off-policy sampling pattern within the \textit{Extended Off-policy Improvement Lower Bound} can enhance training performance.


\begin{figure}[ht]
    \centering
    \includegraphics[width=\textwidth]{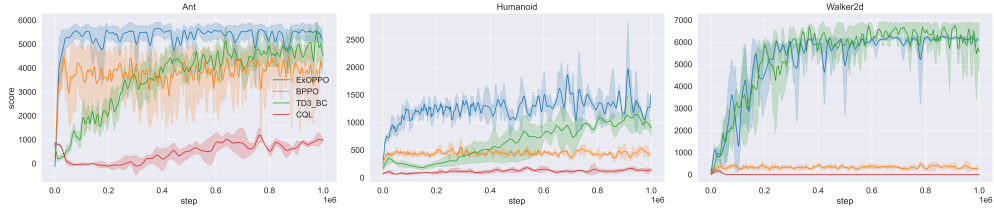}
    \caption{Offline training across MuJoCo tasks for ablation study.}
    \label{fig:offline}
\end{figure}

\subsubsection{Extended Ratio Objective}

In this subsection, we examine whether the extended ratio objective can leverage the shifts in off-policy data while mitigating policy drift. In light of this, we compare the performance of ExO-PPO in the case of offline training with other offline RL algorithms, including Behavior PPO \cite{BPPO}, Conservative Q-Learning \cite{ConOffQ}, and TD3+BC \cite{MinOff} algorithms. As an extreme form of off-policy learning, offline reinforcement learning could lead to more overestimation of those out-of-distribution actions. A common approach to mitigate this issue is to constrain or regularize policy updates. Particularly, on-policy algorithms inherently adopt a conservative objective function for stability \cite{BPPO}. As defined in Section \ref{sec:surrogate}, our extended ratio objective could also be suitable for offline training without any extra constraint. These three comparative offline RL algorithms also exemplify three distinct reinforcement learning paradigms that also follow a similar idea of update regularization, namely stochastic policy gradient \cite{BPPO}, Q-learning \cite{ConOffQ}, and deterministic policy gradient \cite{MinOff}. We evaluated these algorithms across MuJoCo tasks using the Minari dataset \cite{minari}, which hosts a collection of popular Offline Reinforcement Learning datasets, including D4RL \cite{d4rl} benchmarks. Notably, the stochastic policy algorithms, including BPPO and our ExO-PPO, necessitates estimation of the reference policies. Accordingly, we also adopt Gaussian policies as references, where their mean takes the action values from the offline data and its standard deviation decreases over iterations. The results presented in the Figure \ref{fig:offline} demonstrate that our ExO-PPO algorithm attains superior or comparable training performance in most cases of offline training with fewer steps. This indicates that our proposed objective function, the extended ratio objective, is suitable for off-policy or even offline data without additional modifications. 

Overall, through comparative experiments and ablation studies, Exo-PPO, the synergistic combination of extended off-policy improvement lower bound and extended ratio objective together contributes to performance enhancements.

\section{Conclusion}
In this paper, we introduce an Extended Off-policy Proximal Policy Optimization (ExO-PPO) algorithm, which integrates off-policy samples into proximal policy optimization with a novel extended objective. Supported by the extended off-policy policy improvement lower bound, the extended off-policy sampling pattern enables our algorithm to reliably learn from off-policy data. Then, the extended ratio allows for optimizing policy gradients within a broader space while maintaining constraints on divergence. However, it's also shares certain limitations with the original PPO framework, such as sensitive to hyperparameters in various tasks.

For future works, there remain several mathematical ambiguities that require further study. These uncertainties suggest potential areas for investigation such as the impact of distribution shift, the configuration of initial distributions, and the efficiency of gradient propagation, especially in continuous tasks. Exploring these aspects could significantly contribute to the development of a more general reinforcement learning algorithm, highlighting its importance as a future research direction.

\backmatter








\section*{Declarations}


\begin{itemize}
\item Funding 
\item There is no conflict of interests. 
\item We have used the datasets that are publicly available in the literature. Data availability 
\item Our cods will be available at https://github.com/HarriWxy/ExO-PPO if accepted. 
\item Author contribution
\end{itemize}

\begin{appendices}

\section{Proof Details} 
\subsection{Proof of Lemma \ref{lemma:SB}} \label{Approof}
We begin with the stating results from Kakade and Langford \cite{Approx} and Achiam et al. \cite{CPO}. 

\begin{lemma}\label{sec:l1} 
    Consider a reference policy $\pi_\mathrm{ref}$, for any future policy $\pi$ with the same starting state distribution, we have \cite{Approx}
    \begin{equation}
        J(\pi)-J(\pi_\mathrm{ref})=\frac{1}{1-\gamma} \mathop{\mathbb{E}}\limits_{(s,a)\sim \pi} \left[ A^{\pi_\mathrm{ref}}(s,a)  \right]
    \end{equation}
\end{lemma}

\begin{lemma}\label{sec:l3}
    Consider a reference policy $\pi_\mathrm{ref}$ and a future policy $\pi$. Then, the divergence (L1 norm) between the state visitation distributions $d^{\pi_\mathrm{ref}}$ and $d^{\pi}$ is bounded by \cite{CPO}: 
    \begin{equation}
        \Vert d^{\pi} - d^{\pi_\mathrm{ref}} \Vert_1 \le \frac{2\gamma}{1-\gamma} \mathop{\mathbb{E}}\limits_{s\sim d^{\pi_\mathrm{ref}}} [\mathrm{TV}(\pi,\pi_\mathrm{ref})(s)]
    \end{equation}
    where $\mathrm{TV}(\pi,\pi_\mathrm{ref})(s)$ is defined as in Lemma \ref{lemmaLB}.
\end{lemma} 

\begin{proof}
Starting from Lemma \ref{sec:l1}, we add and subtract the term 
$$\frac{1}{1-\gamma} \mathop{\mathbb{E}}\limits_{(s,a)\sim\pi_\mathrm{ref}}\big[r_\mathrm{ref}A^{\pi_\mathrm{ref}}(s,a)\big].$$

Then we have:
\begin{equation}\begin{aligned}
    J(\pi) &-J(\pi_\mathrm{ref}) = \frac{1}{1-\gamma} \mathop{\mathbb{E}}\limits_{(s,a)\sim\pi_\mathrm{ref}}\big[r_\mathrm{ref}A^{\pi_\mathrm{ref}}(s,a)\big] \\
    &+\frac{1}{1-\gamma}\left[ \mathop{\mathbb{E}}\limits_{(s,a)\sim\pi}\big[A^{\pi_\mathrm{ref}}(s,a)\big]-\mathop{\mathbb{E}}\limits_{(s,a)\sim\pi_\mathrm{ref}}\big[r_\mathrm{ref}A^{\pi_\mathrm{ref}}(s,a)\big]\right] \\ 
    &\ge\frac{1}{1-\gamma} \mathop{\mathbb{E}}\limits_{(s,a)\sim\pi_\mathrm{ref}}\big[r_\mathrm{ref}A^{\pi_\mathrm{ref}}(s,a)\big]\\
    &-\frac{1}{1-\gamma}\Big| \mathop{\mathbb{E}}\limits_{(s,a)\sim\pi}\big[A^{\pi_\mathrm{ref}}(s,a)\big]-\mathop{\mathbb{E}}\limits_{(s,a)\sim\pi_\mathrm{ref}}\big[r_\mathrm{ref}A^{\pi_\mathrm{ref}}(s,a)\big]\Big|.
\end{aligned}
\end{equation}

We can restrain the second term using Hölder's inequality:
\begin{equation}\begin{aligned}
    \Big|\mathop{\mathbb{E}}\limits_{s\sim d^{\pi}}\big[\mathop{\mathbb{E}}\limits_{a\sim\pi}[A^{\pi_\mathrm{ref}}(s,a)] & \big]-\mathop{\mathbb{E}}\limits_{s\sim d^{\pi_\mathrm{ref}}}\big[\mathop{\mathbb{E}}\limits_{a\sim\pi}[A^{\pi_\mathrm{ref}}(s,a)]\big]\Big|\\
    &\le \Vert d^{\pi}-d^{\pi_\mathrm{ref}}\Vert_1\big\Vert\mathop{\mathbb{E}}\limits_{a\sim\pi}[A^{\pi_\mathrm{ref}}(s,a)]\big\Vert_\infty
\end{aligned}
\end{equation}

Also note that: 
\begin{equation}\begin{aligned}
    \big\Vert\mathop{\mathbb{E}}\limits_{a\sim\pi}[A^{\pi_\mathrm{ref}}(s,a)]\big\Vert_\infty&=\max_{s\in S}\bigl|\mathop{\mathbb{E}}\limits_{a\sim\pi(\cdot|s)}[A^{\pi_\mathrm{ref}}(s,a)]\bigr|\\
    &=C^{\pi,\pi_\mathrm{ref}}.
\end{aligned}
\end{equation}
    
Thus, we can get Lemma \ref{lemma:SB}, the prerequisite of the theorem \ref{Theo1}, policy improvement lower bound between any future policy $\pi$ and the reference policy $\pi_\mathrm{ref}$: 

\begin{equation}
    \begin{aligned}
        J(\pi)-J(\pi_{\mathrm{ref}}) \ge & \frac{1}{1-\gamma} \mathop{\mathbb{E}}\limits_{(s,a)\sim \pi_{\mathrm{ref}}} \left[ \frac{\pi(a|s)}{\pi_{\mathrm{ref}}(a|s)} A^{\pi_{\mathrm{ref}}}(s,a) \right]\\
        &- \frac{2\gamma C^{\pi,\pi_{\mathrm{ref}}}}{(1-\gamma)^2}\mathop{\mathbb{E}}\limits_{(s,a)\sim \pi_{\mathrm{ref}}}\Big[\mathrm{TV}(\pi,\pi_{\mathrm{ref}})(s)\Big].
    \end{aligned}
\end{equation}
\end{proof}

\subsection{Proof of Theorem \ref{Theo1}} \label{Approoftheo}
\begin{proof}
    As Lemma \ref{lemma:SB} represents an inequality relation between training policy $\pi$ and reference policies $\pi_{t-i}$, $i=0,1,...,M-1$ in the replay buffer, we apply Lemma \ref{lemma:SB} $M$ times in terms of convex combination of sampling probability distribution $\nu=[\nu_0 ,..., \nu_{M-1}]$. 
    \begin{equation}\begin{aligned}
    \sum_{i=0}^{M-1}\nu_i[J(\pi)- J(\pi_{t-i})] \ge & \sum_{i=0}^{M-1}\nu_i \Bigg[\frac{1}{1-\gamma}\Big[ \mathop{\mathbb{E}}\limits_{(s,a)\sim \pi_{t-i}} \big[ r_{t-i}(s,a) A^{\pi_{t-i}}(s,a) \big] \Big]  \\
    & - \Big[\frac{2\gamma C^{\pi,\pi_{t-i}}}{(1-\gamma)^2} \mathop{\mathbb{E}}\limits_{(s,a)\sim \pi_{t-i}}[\mathrm{TV}(\pi,\pi_{t-i})(s) ]\Big]\Bigg]
\end{aligned}\end{equation}

    Taking the expectation with respect to the distribution $\nu$ on both sides of the inequality yields Theorem \ref{Theo1}.
    \begin{equation}\begin{aligned}
    J(\pi)- &\mathop{\mathbb{E}}\limits_{i\sim\nu} [J(\pi_{t-i})] \ge \frac{1}{1-\gamma}\mathop{\mathbb{E}}\limits_{i\sim\nu}\Big[ \mathop{\mathbb{E}}\limits_{(s,a)\sim \pi_{t-i}} \big[ r_{t-i}(s,a) A^{\pi_{t-i}}(s,a) \big] \Big]  \\
    & - \mathop{\mathbb{E}}\limits_{i\sim\nu}\left[\frac{2\gamma C^{\pi,\pi_{t-i}}}{(1-\gamma)^2} \mathop{\mathbb{E}}\limits_{(s,a)\sim \pi_{t-i}}\big[\mathrm{TV}(\pi,\pi_{t-i})(s) \big]\right]
\end{aligned}\end{equation}
    
\end{proof}

\section{Implementation Details} 
We describe more implementation details in this section to facilitate reproduction of our experimental results.

\subsection{Network structure and hyperparameters}
For the experiments on Atari games, the target policy $\pi$ is a Gibbs distribution parameterized by a CNN network with softmax activation. The value approximator $V^{\pi}$ shares a similar network structure with different domain of output. Their base structure and hyperparameters is identical to that of Tianshou \cite{tianshou}, an open-source reinforcement learning library, comprising three convolutional layers followed by two fully connected layers. The initial learning rate is set to $2.5 \times 10^{-4}$, and it would be decayed by a factor of $0.99$ every $5$K training steps. As mentioned above, we set the default parallel environment number to $8$, the episodic run steps per environment to $256$, and the training batch size to 256 for the on-policy algorithms. For mixed off-policy algorithms, we select the number of prior policies $M=4$ and parallel environment number to $2$, ensuring that each sample is calculated the same times as on-policy algorithms. The default clipping parameter is $\varepsilon=0.2$, weight $\beta=1$ for KL divergence, and the edge decaying factor $\alpha$ in ExO-PPO is $5$. As for other similar algorithms, we generally followed the parameter settings in their corresponding papers, including the temperature $\tau = 2$ for P3O-Scopic \cite{TheSuffi}, and the stopping threshold $\delta = 0.25$ for ESPPO \cite{ESPPO}. In addition, we use $\lambda=0.95$ of GAE for advantage estimation, controlling the weighted expectation of long-term return. 

As for the MuJoCo tasks, the target policy $\pi$ is a Gaussian policy, which is a parameterized multivariate Gaussian distribution. The mean is approximated by a dense network, while the standard deviation a separate parameter. For each action, the standard deviation is initialized as a multiple of half of the feasible action range. Observations are standardized using a running mean and standard deviation throughout training. And the learning rate is set to $1.5 \times 10^{-4}$, and it would be decayed by a factor of $0.98$ every $1$M training steps. In terms of TD3 and SAC algorithms, their learning rate is $1 \times 10^{-3}$, and the capacity of replay buffer is $1 \times 10^{6}$. Note that, since we do not have access to the precise action distribution of reference policies in offline training, the mean of these policies is derived from the offline data, while their standard deviation decreases from an initial value of $\frac{1}{\sqrt{2\pi}}$. This initial standard deviation value yields a relatively high probability density at the mean, while maintaining variability in the distribution for continuous actions. In continuous tasks, most of the hyperparameters of ExO-PPO keep the same as the discrete tasks, while the weight for KL divergence $\beta$ changes to $0.1$. 

\subsection{Computational resources}
All experiments were run on an Ubuntu 22.04 operating system, Intel Core i9-10980XE CPU and 128 GB of RAM, and NVIDIA RTX A6000 GPU with 48 GB of dedicated memory. We utilized the open-source platform Gymnasium \cite{gymnasium}, formerly known as OpenAI Gym \cite{gym}, as the simulation environments employed for online training and evaluation. And we leveraged the 'mujoco-expert-v0' dataset from Minari \cite{minari} as our offline training data. The implementations of baseline algorithms, including PPO, TD3, and SAC, were based on the open-source library Tianshou \cite{tianshou}.  

\subsection{Code}
Our cods will be available at https://github.com/HarriWxy/ExO-PPO if accepted.




\end{appendices}


\bibliography{exo}

\end{document}